\documentclass[letterpaper, 10 pt, conference]{ieeeconf}  

\IEEEoverridecommandlockouts                              

\usepackage{amsmath} 
\usepackage{amssymb}  
\usepackage{mathtools}

\newtheorem{theorem}{Theorem}[section]

\newtheorem{lemma}{Lemma}[section]
\newtheorem{definition}{Definition}[section]

\newcommand{\figref}[1]{Fig.~\ref{#1}}
\usepackage{subcaption}

\usepackage[dvipsnames]{xcolor}

\usepackage{siunitx}

\title{Subframework-based Bearing Rigidity Maintenance Control in Multirobot Networks
}

\author{
    J. Francisco Presenza$^{1}$,
    Ignacio Mas$^{2}$,
    J. Ignacio Alvarez-Hamelin$^{3}$, and
    Juan I. Giribet$^{4}$
\thanks{$^{1}$J.F. Presenza is with Universidad de Buenos Aires, Facultad de Ingenier\'ia and CONICET, 
Buenos Aires, Argentina.
    Preferred address for correspondence:   {\tt\small jpresenza@fi.uba.ar}}%
\thanks{$^{2,4}$I. Mas and J. Giribet are with Universidad de San Andr\'es and CONICET, Argentina.
        {\tt\small \{imas,jgiribet\}@udesa.ar}}%
\thanks{$^{3}$J. I. Alvarez-Hamelin is with Universidad de Buenos Aires, Facultad de Ingenier\'ia; and also with CONICET--Universidad de Buenos Aires, INTECIN, Argentina.
        {\tt\small ihameli@fi.uba.ar}}%
\thanks{
© IEEE. Personal use of this material is permitted. Permission from IEEE must be obtained for all other uses, in any current or future media, including reprinting/republishing this material for advertising or promotional purposes, creating new collective works, for resale or redistribution to servers or lists, or reuse of any copyrighted component of this work in other works.
The published version is available at: https://doi.org/10.1109/LCSYS.2025.3580768
} 
}

\begin{document}

\maketitle
\thispagestyle{empty}
\pagestyle{empty}

\begin{abstract}
This work presents a novel approach for \textit{bearing rigidity} analysis and control in multi-robot networks with sensing constraints and dynamic topology. By decomposing the system's framework into \textit{subframeworks}, we express bearing rigidity---a global property---as a set of \textit{local} properties, with rigidity eigenvalues serving as natural \textit{local rigidity measures}. We propose a decentralized gradient-based controller to execute mission-specific commands using only bearing measurements. The controller preserves bearing rigidity by keeping the rigidity eigenvalues above a threshold, using only information exchanged within subframeworks.
Simulations evaluate the scheme's effectiveness, underscoring its scalability and practicality.
\end{abstract}

\section{Introduction}

\label{sec:intro}
An intensive area of research in recent years considers the coordination of multirobot systems based on relative \textit{bearing measurements} (line-of-sight directions).
This is due to the proliferation of missions in GPS-denied environments (e.g., underwater, indoor, etc.), facilitated by low-cost vision-based bearing sensors (e.g., monocular cameras).

Bearing-based \textit{formation control} and \textit{network localization} are widely studied topics in the field \cite{Zhao2019CSM, Michieletto2016}.
Formation controllers rely on bearing information from neighboring robots to guide each agent, enabling the entire team to achieve and maintain a desired spatial geometric pattern \cite{Zhao2016TAC, Schiano2016}. Network localization protocols seek to estimate the relative positions of robots using bearing measurements \cite{Zhao2016AUT, Schiano2016, Spica2016}.
The key enabling property is bearing rigidity (\textit{BR}), which examines whether the available bearings determine the robot positions, up to a translation and a scaling of the framework---mathematical model used to represent the multirobot network.
Recently, a new class of approaches for bearing-based coordination has emerged, grounded in nonlinear observability theory \cite{Tang2023, DeCarli2023}.
These schemes exploit the influence of the control inputs in the system's observability, allowing the rigidity requirement to be relaxed---though they require specific motions to ensure persistence of excitation (\textit{BPE}) conditions.

While the existing literature provides a strong foundation, most works consider a time-invariant graph structure. Or, when changes in the topology are allowed, the observability requirements---\textit{BR} or \textit{BPE}---are generally assumed to be satisfied. 
However, this is idealistic, as the set of available bearings is inherently dynamic due to factors such as limited communication and camera ranges, narrow fields of view, and occlusions.
As a result, \textit{BR} and \textit{BPE} can be lost, leading to failure.
This highlights the need to study bearing observability maintenance as a control problem.
The goal is to enable the multi-robot system to perform arbitrary tasks---such as formation control---using only bearing measurements, despite its dynamic topology. 
This is accomplished by ensuring that trajectories preserve the observability, which in turn allows the robots to estimate their relative positions.

{
In this letter, we address \textit{bearing rigidity maintenance} for a team of point robots in $\mathbb{R}^d$ ($d \geq 2$), each equipped with an onboard camera of limited range and field of view for acquiring bearing measurements.
Robot positions are estimated from bearings and controlled via velocity commands, while camera orientations—defined by a yaw angle—are known and controllable through angular velocity inputs.
Yaw angles are used to model camera orientations and are not part of the uncertain robot state.
By knowing these angles w.r.t a common reference frame, robots can express bearing measurements in that frame, allowing the use of rigidity theory in $\mathbb{R}^d$ \cite{Zhao2019CSM, Zhao2016TAC}.
We recognize that absolute orientation sensors (e.g., compasses) may be unreliable in certain environments, and that yaw angles can alternatively be inferred from bearing measurements using rigidity in $\mathbb{R}^d \times \mathcal{S}^1$ \cite{Schiano2016, Spica2016}.
However, when possible, estimating robot orientations independently prevents overconstraining the
sensing topology, as rigidity in $\mathbb{R}^d$ typically requires fewer measurements.}

A bearing rigidity maintenance scheme was proposed by \cite{Schiano2017}, building upon the earlier approach presented in \cite{Zelazo2015} which considers distance measurements. 
In these strategies, the loss of rigidity is treated as an obstacle in the space to be avoided.
By defining a measure of a network's degree of rigidity, the \textit{rigidity eigenvalue}, the authors proposed a gradient descent controller that guides the robots to keep this
measure above a minimum level, while enabling changes in the network topology.
In \cite{Zelazo2015}, the robots' communication ranges determine the acquisition of distance measurements, and \cite{Schiano2017} incorporates limited camera field-of-view and occlusions between robots.

Scalable decentralized rigidity maintenance remains an open problem.
This is because of the global nature of rigidity: the loss of a measurement---such as when a robot moves out of another's field-of-view---can compromise the rigidity of the entire network.
As a consequence, developing fully decentralized algorithms is challenging.
The approaches referenced above require each robot to access the rigidity eigenvalue and its eigenvector, which depend on the entire network.
For the distance case, \cite{Zelazo2015} proposed a consensus filter that allows each agent to locally estimate these global variables. 
Nevertheless, the convergence properties of such estimators are strongly influenced by the network's
diameter, which serves as a proxy for communication delay.
As the number of robots increases, the diameter is expected to grow due to the limited number of connections each robot can maintain, because of energy and computing constraints.
This limits the scalability of consensus-based rigidity maintenance controllers.

In this work, we address these challenges by proposing a decomposition of a framework into smaller groups called \textit{subframeworks}.
Based on this, we state a necessary and sufficient condition for bearing rigidity in $\mathbb{R}^d$ in terms of subframeworks, enabling its representation as a set of \textit{local} conditions.
The rigidity eigenvalues associated with the subframeworks arise naturally as \textit{local rigidity measures}. 
This allows the use of subframeworks as a novel method for decentralized bearing rigidity maintenance in networks with dynamic topology.
A subframework-based approach was proposed for the distance case in \cite{PresenzaACC2022}. To our knowledge, this is the first application of the technique to the bearing setting.
The main contributions of this work are: (a) the development of tools to express bearing rigidity through subframeworks; (b) a controller that maintains rigidity by keeping each subframework’s rigidity eigenvalue above a threshold, achieving a global objective without the need for network-wide exchange of information; and (c) the introduction of performance metrics to assess its scalability and practicality, and to enable comparison with existing methods.

This letter is organized as follows. Section \ref{sec:preliminaries} presents basic definitions and an overview of bearing rigidity theory in $\mathbb{R}^d$. In Section \ref{sec:subframework}, we introduce a subframework-based bearing rigidity formalism and include a comparison with the distance-based setting \cite{PresenzaACC2022}. 
The subframework-based rigidity maintenance controller is introduced in Section \ref{sec:control}, together with a discussion on its decentralized implementation and its scalability. Simulation results are presented in Section \ref{sec:simulations}, and Section \ref{sec:conclusions} provides final conclusions and directions for future work.

\section{Preliminaries}
\label{sec:preliminaries}

\subsection{Basic Definitions}

A stack of $n$ vectors $v_1, \ldots, v_n$ is considered as a column vector and represented as $(v_i)_{i \in \mathcal{I}}$ given $\mathcal{I} = \{1, \ldots, n\}$.
The vector of all ones is $1_d \in\mathbb{R}^d$, and the identity matrix $I_d \in \mathbb{R}^{d \times d}$.
The null space of $A$ is $\mathrm{null}(A)$.
The Kronecker product is symbolized as $A \otimes B$.
Let $\mathcal{G} = (\mathcal{V}, \mathcal{E})$ be a graph with vertex set $\mathcal{V} = \{1, \ldots, n\}$ and edge set $\mathcal{E} \subseteq \mathcal{V} \times \mathcal{V}$.
The neighbors of $i$ are $\mathcal{N}_i = \{j : (i, j) \in \mathcal{E}\}$.
A graph is said to be undirected if $(i, j) \in \mathcal{E} \Leftrightarrow (j, i) \in \mathcal{E}$, and directed otherwise.
Undirected edges are written as $\{i, j\}$. 
The distance $\delta_{ij}\!=\!\delta_{ji}$ between $i$ and $j$ in an undirected graph is the length of the shortest path between them, measured in number of edges.
The maximum distance $\Delta(\mathcal{G}) =\max_{ij} \delta_{ij}$ is called the diameter.
A $d$-dimensional framework $\mathcal{F} = (\mathcal{G}, \boldsymbol{p})$ is given by an undirected graph $\mathcal{G}$ and an injective realization in $\mathbb{R}^d$ represented by $\boldsymbol{p} = (p_i)_{i \in \mathcal{V}} \in \mathbb{R}^{d|\mathcal{V}|}$.

\subsection{{Robot and Network Models}}
\label{sec:robot_network_model}

{
The state of each robot is modeled as $(p_i, \psi_i) \in \mathbb{R}^d \times \mathcal{S}^1$, where $p_i$ denotes its position and $\psi_i$ the camera's yaw angle.
Cameras feature limited range $\ell_i\!>\!0$ and field of view characterized by a half-angle cosine $0\!<\!\gamma_i\!<\!1$.
Let $b_{ij} = (p_j - p_i)/d_{ij}$ be the bearing measurement that $i$ takes from $j$ and $n_i$ the camera's optical axis, both expressed in the common reference frame, where $d_{ij} = \Vert p_i - p_j\Vert$.
Then, the (\emph{directed}) sensing graph $\mathcal{G}_s = (\mathcal{V}, \mathcal{E}_s)$ is defined such that $(i, j) \in \mathcal{E} \Leftrightarrow d_{ij} \leq \ell_i \text{ and } n_i^\mathsf{T} b_{ij} \geq \gamma_i$.
The sensing graph is a function of the joint state $(\boldsymbol{p}, \boldsymbol{\psi})$ where $\boldsymbol{p} = (p_i)_{i \in \mathcal{V}}$ and $\boldsymbol{\psi} = (\psi_i)_{i \in \mathcal{V}}$; however, only $\boldsymbol{p}$ has to be estimated from bearing measurements, as $\boldsymbol{\psi}$ is known.
Observe that $b_{ij} = - b_{ji}$ for all $(i, j) \in \mathcal{E}_s$, meaning that reciprocal measurements carry the same information.
Therefore, the sensing graph can be replaced by its undirected version, which is denoted $\mathcal{G} = (\mathcal{V}, \mathcal{E})$ for simplicity.
Robots are also able to communicate with each other over a radio channel within a range $\ell_c$ ($\geq\!\ell_i$).
The (\textit{undirected}) communication graph $\mathcal{G}_c = (\mathcal{V}, \mathcal{E}_c)$ depends solely on $\boldsymbol{p}$, and has $\{i, j\} \in \mathcal{E}_c \Leftrightarrow d_{ij} \leq \ell_c$.}

\subsection{Bearing Rigidity}

This section provides an overview of bearing rigidity theory in $\mathbb{R}^d$ for undirected sensing topologies, see \cite{Zhao2019CSM, Zhao2016AUT}.

\begin{definition}
    Let $(\mathcal{G}, \boldsymbol{p})$ be a $d$-dimensional framework. The \textit{bearing function} $b_{\mathcal{G}}: \mathbb{R}^{d|\mathcal{V}|} \to \mathbb{R}^{d|\mathcal{E}|}$ stacks the bearings associated with each edge, $b_{\mathcal{G}}(\boldsymbol{p}) = (b_{ij})_{\{i, j\} \in \mathcal{E}}$.
\end{definition}

The bearing function is invariant to translations and uniform scalings of the framework, known as \textit{trivial motions}.
Rigidity theory studies whether the set of available bearings determines $\boldsymbol{p}$ up to a trivial motion.
To express this, define the $d|\mathcal{E}|\!\times\!d|\mathcal{V}|$ Jacobian of $b_{\mathcal{G}}(\boldsymbol{p})$ with respect to $\boldsymbol{p}$, known as the \textit{bearing rigidity matrix}, denoted $R_{\mathcal{G}}(\boldsymbol{p})$.
Then,
\begin{equation}
    T(\boldsymbol{p}) \coloneqq \mathrm{span}\{1_{|\mathcal{V}|} \otimes I_d, \boldsymbol{p}\} \subseteq \mathrm{null}(R_{\mathcal{G}}(\boldsymbol{p})),    
\end{equation}
where $T(\boldsymbol{p})$ is known as the space of \textit{infinitesimal trivial motions} and its dimension is $d+1$.

\begin{definition}
    A $d$-dimensional framework $(\mathcal{G}, \boldsymbol{p})$ is said to be \textit{infinitesimally bearing rigid} (\textit{IBR}) if $\mathrm{null}(R_{\mathcal{G}}(\boldsymbol{p})) = T(\boldsymbol{p})$ (equivalently, if $\mathrm{rank}(R_{\mathcal{G}}(\boldsymbol{p})) = d|\mathcal{V}| - d - 1$).
    \label{def:inf_rigidity}
\end{definition}
\begin{definition}
    Let $(\mathcal{G}, \boldsymbol{p})$ be a $d$-dimensional framework, and let $w_{ij} = w_{ji} > 0$ be a weight assigned to each edge. 
    The $d|\mathcal{V}| \times d|\mathcal{V}|$ \textit {weighted bearing Laplacian} is defined such that the $(i, j)$-th $d \times d$ block has the form
    \begin{equation}
        [B_{\mathcal{G}}(\boldsymbol{p})]_{i, j} = 
        \begin{cases}
            0, & \ i \neq j, \{i, j\} \notin \mathcal{E} \\
            -w_{ij} P_{ij}, & \ i \neq j, \{i, j\} \in \mathcal{E} \\
            \sum_{k \in \mathcal{N}_i} w_{ik} P_{ik}, & \ i=j
        \end{cases}.
        \label{eq:bearing_laplacian}
    \end{equation}
    Here, $P_{ij} = P_{ji} = I_d - b_{ij} b_{ij}^\mathsf{T}$ is an orthogonal projection matrix associated with each edge. 
    \label{def:bearing_laplacian}
\end{definition}

{
With a suitable choice of weights (see Section \ref{sec:cost_rm}), \eqref{eq:bearing_laplacian} can be used to model a sensing graph without reciprocal bearing measurements.}
It holds that $B_{\mathcal{G}}(\boldsymbol{p}) = R_{\mathcal{G}}(\boldsymbol{p})^\mathsf{T} (W \otimes I_d) R_{\mathcal{G}}(\boldsymbol{p})$ where
$W$ is a ${|\mathcal{E}| \times |\mathcal{E}|}$ diagonal matrix whose nonzero entries are $w_{ij} d_{ij}^2$ for each $\{i, j\} \in \mathcal{E}$.
The bearing Laplacian is positive semi-definite and its eigenvalues are arranged in increasing order.
Rigidity is characterized by the $(d+2)$-th smallest eigenvalue of $B_{\mathcal{G}}(\boldsymbol{p})$, referred to as the \textit{rigidity eigenvalue}.
This eigenvalue also serves as a metric for a framework's rigidity level, making it useful for designing rigidity controllers.

\begin{theorem}[See {\cite{Schiano2017}}]
    Let $(\mathcal{G}, \boldsymbol{p})$ be a $d$-dimensional framework. Then, $(\mathcal{G}, \boldsymbol{p})$ is \textit{IBR} if and only if the \textit{rigidity eigenvalue} $\lambda_{d+2}(B_{\mathcal{G}}(\boldsymbol{p})) > 0$.
    \label{th:inf_rigidity}
\end{theorem}

The following lemma, used in Section \ref{sec:subframework}, follows directly from Definition \ref{def:bearing_laplacian} and basic algebraic manipulation.
\begin{lemma}
    Let $(\mathcal{G}, \boldsymbol{p})$ be a $d$-dimensional framework and $\boldsymbol{u} = (u_i)_{i \in \mathcal{V}} \in \mathbb{R}^{d|\mathcal{V}|}$. 
    Then, 
    \begin{enumerate}
        \item $\boldsymbol{u}^\mathsf{T} B_{\mathcal{G}}(\boldsymbol{p}) \boldsymbol{u} = \textstyle{\sum_{\{i, j\} \in \mathcal{E}}} w_{ij} (u_i - u_j)^\mathsf{T} P_{ij} (u_i - u_j)$;
        \item $B_{\mathcal{G}}(\boldsymbol{p}) \boldsymbol{u} = 0 \Leftrightarrow P_{ij} (u_i - u_j) = 0$ for all $\{i, j\} \in \mathcal{E}$.
    \end{enumerate}
    \label{lem:laplacian_prop}
\end{lemma}

\section{Subframework-based Bearing Rigidity Theory}
\label{sec:subframework}

\subsection{Subframeworks}
The primary challenge in implementing rigidity maintenance controllers is the inherently global nature of this property: removing even a single edge may cause a framework to lose its rigidity.
In consequence, to enforce rigidity maintenance, each robot must engage in a network-wide information exchange to obtain global parameters such as the rigidity eigenvalue.
This can be problematic, especially as the network size increases.
For this reason, adapting the study of rigidity to account for subsets of nearby robots forming rigid \textit{subframeworks} is a natural step.
To mitigate the combinatorial complexity associated with the choice of these subframeworks, we restrict the analysis to subsets formed by a vertex and those within a fixed distance (in number of edges).
\begin{definition}
    Let $\mathcal{F}$ be a framework, $i \in \mathcal{V}$ and $r \in \mathbb{N}$.
    The \textit{ball subframework} of radius $r$ with center $i$, denoted $\mathcal{F}_{i, r}$, is the one having vertex set $\mathcal{V}_{i, r} = \{j \in \mathcal{V} : \delta_{ij} \leq r\}$, edge set $\mathcal{E}_ {i, r} = \mathcal{E} \cap (\mathcal{V}_{i, r} \times \mathcal{V}_{i, r})$, and realization $\boldsymbol{p}_{i, r} = (p_j)_{j \in \mathcal{V}_{i, r}}$.
    \label{def:ball_subframework}
\end{definition}

Ball subframeworks proved to be convenient to design decentralized control strategies in the distance rigidity setting.
In \cite{Amani2020} $1$-hop balls ($r = 1$) were used to implement a rigidity recovery protocol.
In \cite{PresenzaACC2022}, the definition was extended to balls of arbitrary radii, allowing its application to arbitrary connected frameworks.

\subsection{A New Approach to Bearing Rigidity}
In this section, we present a subframework-based necessary and sufficient condition for bearing rigidity.
We start by considering the union of frameworks with common vertices.
Consider two frameworks $\mathcal{F} = (\mathcal{G}, \boldsymbol{p})$, $\mathcal{F}' = (\mathcal{G}', \boldsymbol{p}')$ such that $p_i = p'_i$ for each $i \in \mathcal{V} \cap \mathcal{V}'$.
Then $\mathcal{F} \cup \mathcal{F}'$ is the framework having vertex set $\mathcal{V} \cup \mathcal{V}'$, edge set $\mathcal{E} \cup \mathcal{E}'$, and realization $(p_i)_{i \in \mathcal{V} \cup \mathcal{V}'}$.

\begin{lemma}
    Let $\mathcal{F}$ and $\mathcal{F}'$ be two infinitesimally bearing rigid frameworks in $\mathbb{R}^d$ with $|\mathcal{V}|, |\mathcal{V}'| \geq 2$.
    Then, $\mathcal{F} \cup \mathcal{F}'$ is infinitesimally bearing rigid if and only if $|\mathcal{V} \cap \mathcal{V}'| \geq 2$.
    \label{lem:subframework_union}
\end{lemma}
\begin{proof}
    \textit{Sufficiency.} Notice that if $\mathcal{V} \cap \mathcal{V}' = \emptyset$, then $\mathcal{F} \cup \mathcal{F}'$ is not connected, thus not \textit{IBR}.
    On the other hand, suppose that $\mathcal{V} \cap \mathcal{V}' = \{a\}$.
    Define $\boldsymbol{u} = (u_i)_{\in \mathcal{V} \cup \mathcal{V}'}$ such that $u_i = 0$ if $i \in \mathcal{V}$ and $u_i = p_i - p_{a}$ if $i \in \mathcal{V}'$.
    Observe that $P_{ij}(u_i - u_j) = 0$ for all $\{i, j\} \in \mathcal{E} \cup \mathcal{E}'$; hence, $\boldsymbol{u} \in \mathrm{null}(B_{\mathcal{G} \cup \mathcal{G}'}(\boldsymbol{p}))$.
    Also, that $\boldsymbol{u} \notin T(\boldsymbol{p})$.
    Therefore $\mathcal{F} \cup \mathcal{F}'$ is not \textit{IBR}. 
    \textit{Necessity.} Assume that $|\mathcal{V} \cap \mathcal{V}'| \geq 2$. Consider $\boldsymbol{u} = (u_i)_{i \in \mathcal{V}}$ such that $B_{\mathcal{G} \cup \mathcal{G}'}(\boldsymbol{p}) \boldsymbol{u} = 0$. 
    From Lemma \ref{lem:laplacian_prop}, it follows that $B_{\mathcal{G}}(\boldsymbol{p}) \boldsymbol{u} = B_{\mathcal{G}'}(\boldsymbol{p}) \boldsymbol{u} = 0$.
    Since $\mathcal{F}, \mathcal{F}'$ are \textit{IBR}, then there are scalars $\alpha, \alpha'$ and vectors $v, v' \in \mathbb{R}^d$ such that $u_i = \alpha p_i + v$ if $i \in \mathcal{V}$ and $u_i = \alpha' p_i + v'$ if $i \in \mathcal{V}'$.
    Consider distinct $a, b \in \mathcal{V} \cap \mathcal{V}'$.
    Then, it must hold that $u_a = \alpha p_a + v = \alpha' p_a + v', u_b = \alpha p_b + v = \alpha' p_b + v'$ which leads to the system of equations
    \begin{equation*}
        (\alpha - \alpha') p_a  + (v - v') = (\alpha - \alpha') p_b  + (v - v') = 0.
    \end{equation*}
    If $\alpha \neq \alpha'$ then $p_a = p_b$, contradicting the injective realization hypothesis. Hence $\alpha = \alpha'$, which implies $v = v'$.
    Therefore $\boldsymbol{u} \in T(\boldsymbol{p})$. We conclude that $\mathcal{F} \cup \mathcal{F}'$ is \textit{IBR}.
\end{proof}

\begin{definition}
    Let $\mathcal{F}$ be a $d$-dimensional framework and $i \in \mathcal{V}$.
    The \textit{minimal radius} associated with $i$ is
    \begin{equation}
        r^*_i = \inf \{r \in \mathbb{N} : \lambda_{d+2}(\mathcal{F}_{i, r}) > 0\}.
        \label{eq:minimal_radius}
    \end{equation}
    If $r^*_i$ is finite, $\mathcal{F}_{i, r^*_i}$ is called the \textit{minimal ball subframework}.
    \label{def:minimal_radius}
\end{definition}

For the rest of this paper, let us denote $\mathcal{F}_{i, r^*_i}$ simply as $\mathcal{F}_i$.
When $r^*_i$ is finite, $\mathcal{F}_i$ corresponds to the smallest infinitesimally rigid ball subframework centered at $i$.

\begin{theorem}[Subframework-based Rigidity]
    Let $\mathcal{F}$ be a connected $d$-dimensional framework. Then, $\mathcal{F}$ is infinitesimally bearing rigid if and only if $r^*_i$ \eqref{eq:minimal_radius} is finite for all $i \in \mathcal{V}$.
    \label{th:subframework_based_rigidity}
\end{theorem}
\begin{proof}
\textit{Sufficiency.} Since $\mathcal{F}$ is connected, for all $i$ there exists $h$ such that $\mathcal{F}_{i, h} = \mathcal{F}$, and $\lambda_{d+2}(\mathcal{F})> 0$ by hypothesis, hence $r^*_i \leq h$ holds.
\textit{Necessity.} 
By hypothesis, the minimal ball subframework $\mathcal{F}_{i} = (\mathcal{G}_i, \boldsymbol{p}_i)$ is \textit{IBR} for each $i \in \mathcal{V}$.
We prove that $\mathcal{F}$ is \textit{IBR} by the following argument.
First, note that given two adjacent vertices $i, j$, it holds that $|\mathcal{V}_i|, |\mathcal{V}_j| \geq 2$ and that $|\mathcal{V}_i \cap \mathcal{V}_j| \geq 2$ since $i, j \in \mathcal{V}_i \cap \mathcal{V}_j$.
By Lemma \ref{lem:subframework_union} it follows that $\mathcal{F}_i \cup \mathcal{F}_j$ is \textit{IBR}.
Since $\mathcal{F}$ is connected, repeating this process for each edge in a spanning tree results in a framework that is \textit{IBR} and equal to $\mathcal{F}$.
\end{proof}

\subsection{Comparison with Distance Rigidity}

Distance rigidity theory is a well-established field; see, for instance, \cite[Chap. 61-63]{Goodman2017}.
However, infinitesimal distance rigidity (\textit{IDR}), which can be tested using the distance rigidity matrix, is not sufficient for distance-based localizability.
According to \cite{Zhao2019CSM}, bearing rigidity offers advantages because each measurement imposes $(d-1)$ constraints, compared to $1$ per distance measurement. For $d=2$, infinitesimal distance and bearing rigidity are equivalent, but not for $d \geq 3$. Theorem \ref{th:idr_implies_ibr} clarifies their relationship in higher dimensions.

\begin{theorem}
    Let $\mathcal{F} = (\mathcal{G}, \boldsymbol{p})$ be a $d$-dimensional framework. If $\mathcal{F}$ is infinitesimally \textit{distance} rigid, then, with probability $1$, it is also infinitesimally \textit{bearing} rigid.
    \label{th:idr_implies_ibr}
\end{theorem}
\begin{proof}
    Since $(\mathcal{G}, \boldsymbol{p})$ is \textit{IDR} in $\mathbb{R}^d$, then there is a realization $\boldsymbol{p}' \in \mathbb{R}^2$, such that $(\mathcal{G}, \boldsymbol{p}')$ is also \textit{IDR}.
    (see \cite[Th. 63.2.11]{Goodman2017}).
    In $\mathbb{R}^2$, \textit{IDR} and \textit{IBR} are equivalent \cite[Th. 8]{Zhao2016TAC}.
    Moreover, since $(\mathcal{G}, \boldsymbol{p}')$ is \textit{IBR} in $\mathbb{R}^2$, there is a realization $\boldsymbol{p}'' \in \mathbb{R}^d$, for any $d \geq 3$, such that $(\mathcal{G}, \boldsymbol{p}'')$ is \textit{IBR} \cite[Th. 7]{Zhao2016TAC}.
    Finally, if $(\mathcal{G}, \boldsymbol{p}'')$ is \textit{IBR}, then it is so for ``almost all'' realizations, i.e., for a dense open set of realizations whose complement has measure zero \cite[Lem. 2]{Zhao2017}.
\end{proof}

Theorem \ref{th:idr_implies_ibr} establishes that the minimal ball subframeworks using bearings are `almost surely' at most as large as those based on distances. 
To validate this, we simulated random frameworks using the $(n, \rho)$-Erd\H{o}s R\'{e}nyi model.
We chose $\rho = 5 / (n-1)$ which induces an expected average vertex degree of $5$.
The vertex positions were sampled from a uniform distribution within the unit cube of $\mathbb{R}^3$.
\figref{fig:distance_vs_bearing} illustrates the average results obtained by sampling $50$ frameworks that are both \textit{IDR} and \textit{IBR} per $n$.
It follows that bearings induce smaller subframeworks: for distance-based rigidity, almost all $i$ have $r^*_i \leq 3$, and $r^*_i \leq 2$ was dominant ($\geq 50 \%$) up to $n \approx 25$; for bearing rigidity, $r^*_i \leq 2$ remained dominant ($\geq 50 \%$) up to $n \approx 50$.
This behavior was also observed for other average degree values, though results are omitted for brevity.

\begin{figure}[!tb]
    \centering
    \vspace*{5pt}
    \begin{subfigure}{0.49\columnwidth}
        \includegraphics[scale=0.85]{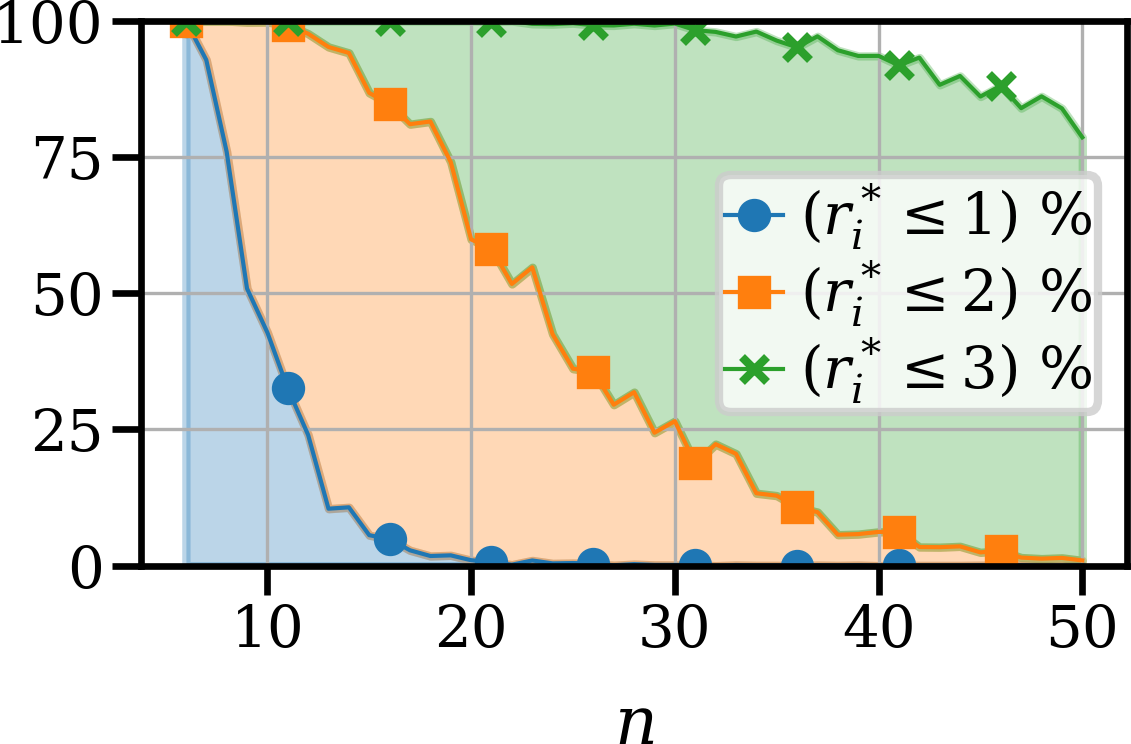}
        \caption{Distance}
        \label{fig:distance_extents}
    \end{subfigure}
    \begin{subfigure}{0.49\columnwidth}
        \includegraphics[scale=0.85]{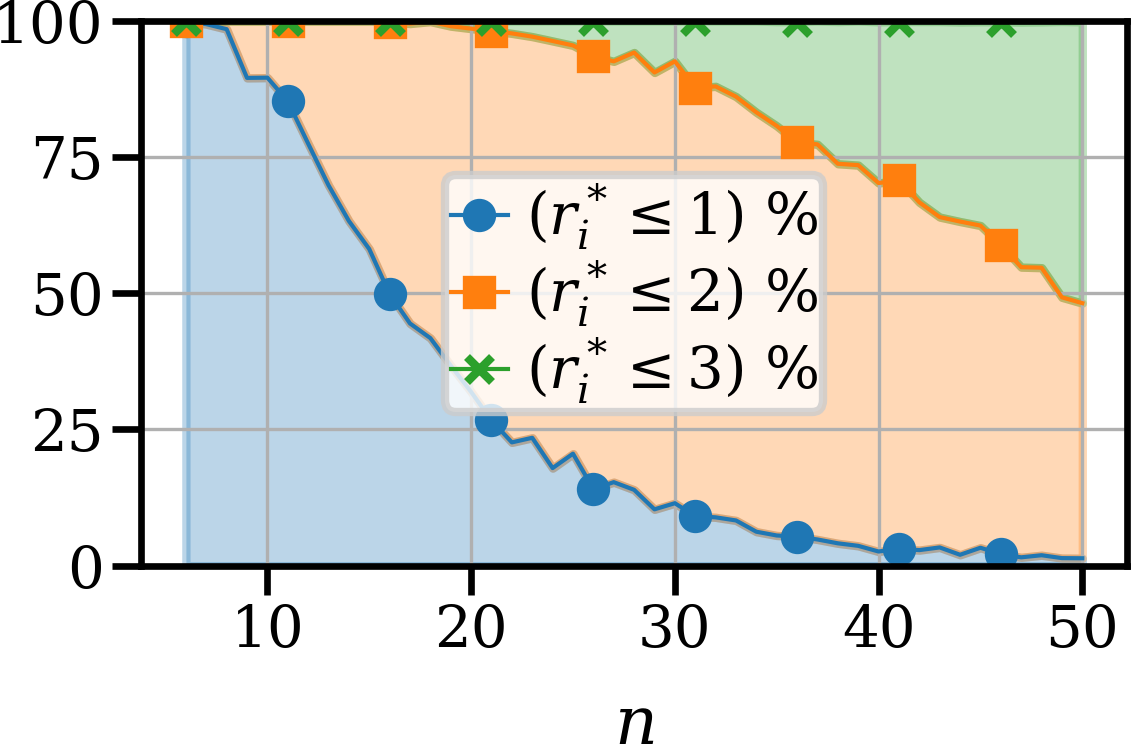}
        \caption{Bearing}
        \label{fig:bearing_extents}
    \end{subfigure}
    \caption{Percentage of subframeworks with minimal radius less than or equal to $k = 1, 2, 3$, versus the number of nodes, $n$.}
    \label{fig:distance_vs_bearing}
\end{figure}

\section{Decentralized Rigidity Maintenance Control}
\label{sec:control}

This section presents a decentralized bearing rigidity maintenance controller that leverages the proposed subframework formalism.
This scheme considers robots in $\mathbb{R}^d \times \mathcal{S}^1$, equipped with onboard cameras featuring limited range and field of view. 
The control objective is to allow a team of mobile robots to carry out arbitrary missions using bearing measurements.
To ensure that robots can reconstruct their relative positions, the controller must constrain the mission-specified trajectories to preserve bearing rigidity.

The proposed strategy is based on the minimization of cost functions defined over ${(\boldsymbol{p}, \boldsymbol{\psi})}$.
Namely, (a) $J_\mathrm{m}$ associated with the successful implementation of the desired mission (e.g., trajectory following), (b) $J_\mathrm{c}$ which addresses collision avoidance between robots, and (c) $J_\mathrm{r}$ to ensure rigidity maintenance.

\subsection{Collision Avoidance Cost Function}
To achieve collision avoidance, we consider neighboring agents in the communication graph $\mathcal{G}_c$, which guarantees the detection of nearby robots.
The chosen cost function is
\begin{equation}
    J_\mathrm{c}(\boldsymbol{p}) = \sum_{\{i, j\} \in \mathcal{E}_c} \, ((d_{ij} - \ell_c\,)/(d_{ij} - \ell_0))^2,
\label{eq:cost_ca}
\end{equation}
where $\ell_0$ ($< \ell_c\,$) indicates the minimum allowed distance.
This cost vanishes (with a vanishing derivative) as two neighbors approach the communication range.

\subsection{Rigidity Maintenance Cost Function}
\label{sec:cost_rm}
To incorporate into the control scheme the dependency of the sensing graph on ${(\boldsymbol{p}, \boldsymbol{\psi})}$, weights assigned to edges are used \eqref{eq:bearing_laplacian}.
These are determined via the sigmoid function
\begin{equation}
    \sigma_{s, m}(x) = (1 + \exp(s \cdot (m - x)))^{-1},
\end{equation}
where $s > 0$ adjusts the rate of decay, and $m$ the curve's midpoint.
Then, for each pair $i, j$ we set $w_{ij} = w_{R_{ij}} w_{F_{ij}} + w_{R_{ji}} w_{F_{ji}}$ where $w_{R_{ij}} = 1 - \sigma_{10, 0.9\ell_i}(d_{ij})$ and $w_{R_{ji}} = 1 - \sigma_{10, 0.9 \ell_j}(d_{ij})$ account for the limited sensing range; and $w_{F_{ij}} = \sigma_{40, 1.2\gamma_i}(n_i^\mathsf{T} b_{ij})$ and $w_{F_{ji}} = \sigma_{40, 1.2 \gamma_j}(n_j^\mathsf{T} b_{ji})$ consider the limited field of view.
Thus, $w_{ij}$ approaches 2 when both robots are within each other's field of view, $1$ when only one is, and approaches $0$ as both lie outside it.

Consider the sensing framework $(\mathcal{G}, \boldsymbol{p})$. For each $i \in \mathcal{V}$, let $\mathcal{F}_i = (\mathcal{G}_i, \boldsymbol{p}_i)$ be the minimal ball subframework, $\boldsymbol{B}_i \coloneqq B_{\mathcal{G}_i}(\boldsymbol{p}_i)$ its bearing Laplacian, and $\boldsymbol{\lambda}_{i, k} \coloneqq \lambda_{k}(\boldsymbol{B}_i)$ its $k$-th eigenvalue.
Observe that $\mathcal{G}_i$ is a function of the subframework state $(\boldsymbol{p}_i, \boldsymbol{\psi}_i)$.
Guided by Theorem \ref{th:subframework_based_rigidity}, to maintain the rigidity of the framework, the controller preserves the rigidity of its subframeworks.
As discussed in \cite[III-D]{Schiano2017}, the rigidity eigenvalue's gradient is not well defined when its multiplicity is greater than one.
To avoid this indefiniteness, we propose cost function \eqref{eq:cost_rm}, that uses all nonzero eigenvalues of each $\boldsymbol{B}_i$.
This approach aligns with the well-known D-optimality criterion, which aims to maximize the system's observability.
\begin{equation}
    J_\mathrm{r}{(\boldsymbol{p}, \boldsymbol{\psi})} = \sum_{i \in \mathcal{V}} J_{\mathrm{r}, i}, \ J_{\mathrm{r}, i} = - \!\log \prod_{k=d+2}^{d|\mathcal{V}_i|} (\boldsymbol{\lambda}_{i, k} - \lambda_0).
    \label{eq:cost_rm}
\end{equation}
Thus, \eqref{eq:cost_rm} becomes unbounded if and only if there is an $i$ such that its rigidity eigenvalue ($\boldsymbol{\lambda}_{i, d+2}$) approaches a minimum threshold ($\lambda_0\!>\!0$), which occurs when subframework $\mathcal{F}_i$ gets close to losing rigidity.
Note that, under the communication protocol discussed in Section \ref{sec:desc_impl}, employing $\boldsymbol{\lambda}_{i, k}$ for $k=d+2, \ldots, d |\mathcal{V}_i|$ incurs no additional communication complexity compared to using only $\boldsymbol{\lambda}_{i, d+2}$.

\subsection{Control Action}
As is common in the literature, we adopt the single-integrator dynamic model $(\dot{p}_i, \dot{\psi}_i) = (v_i, \omega_i)$, where the right-hand side represents the control input.
The control action is then defined as an anti-gradient, designed to guide the group of robots toward a local minimum of the total cost $J{(\boldsymbol{p}, \boldsymbol{\psi})} = \kappa_\mathrm{m} J_\mathrm{m} + \kappa_\mathrm{c}  J_\mathrm{c} + \kappa_\mathrm{r}  J_\mathrm{r}$. Scalars $\kappa_\mathrm{m}, \kappa_\mathrm{c}, \kappa_\mathrm{r} > 0$ are included to adjust the relative importance of each term.
Hence, $v_i = - \partial J / \partial p_i$ and $\omega_i = - \partial J / \partial \psi_i$, where
\begin{align}
    \frac{\partial J_\mathrm{c}}{\partial p_i} &= 2 (\ell_c - \ell_0)  \sum_{j \in \mathcal{N}_{c, i}} \frac{\ell_c - d_{ij}}{(d_{ij} - \ell_0)^3} b_{ij}, \quad \frac{\partial J_\mathrm{c}}{\partial \psi_i} = 0  \label{eq:grad_ca} \\
    \frac{\partial J_\mathrm{r}}{\partial x_i} &= \sum_{j \in \mathcal{I}_i} \nu_{ji}, \ \nu_{ji} =\!\sum_{k=d+2}^{d|\mathcal{V}_j|}\! \frac{1}{\lambda_0 - \boldsymbol{\lambda}_{j, k} } \boldsymbol{\upsilon}_{j, k}^\mathsf{T} \frac{\partial \boldsymbol{B}_j}{\partial x_i} \boldsymbol{\upsilon}_{j, k},
    \label{eq:grad_rm}
\end{align}
where $x_i \coloneqq (p_i, \psi_i)$, $\mathcal{I}_i = \{j : \delta_{ij} \leq r^*_j\}$ is the set of vertices whose subframeworks include $i$, 
and $\boldsymbol{\upsilon}_{j, k}$, $k = d+2, \ldots, d|\mathcal{V}_j|$ are a set of mutually orthogonal and unitary eigenvectors, each associated with $\boldsymbol{\lambda}_{j, k}$.
A closed-form expression can be obtained by following the approach in \cite{Schiano2017}, here omitted for brevity.

\subsection{Decentralized Implementation}
\label{sec:desc_impl}
{
The rigidity control action \eqref{eq:grad_rm} requires robot $i$ to know $\nu_{ji}$ for all $j \in \mathcal{I}_i$.
The term $\nu_{ji}$ represents a velocity that $i$ should follow to preserve the rigidity of $\mathcal{F}_j$.
The following routing protocol ensures that each robot obtains this information.
Consider a subframework $\mathcal{F}_j$ containing robot $i$.
If $j$ receives the state $x_l = (p_l, \psi_l)$ for all $l \in \mathcal{V}_j$, then it can construct $\boldsymbol{B}_j$, compute $\nu_{ji}$ and send it over $\delta_{ji}$ hops reaching $i$.
For this purpose, each $i \in \mathcal{V}$ emits $x_i$ across $q_i = \max_{j \in \mathcal{I}_i} \delta_{ij}$ hops.}

{
We now discuss the communication requirements of the proposed protocol.
On the first place, the worst-case round-trip delay for the $j$-th subframework is $H_j = 2 r^*_j$.
On the second place, to estimate the communication complexity, we proceed as follows.
To complete a communication round, robot $i$ must transmit:
(a) state $x_l$ for all $l$ such that $\delta_{li} < q_l$, and (b) action $v_{jl}$ for all $j$ and all $l$ such that $\delta_{ji} < \delta_{jl} \leq r^*_j$.
Hence, we define $C_i = |\{l : \delta_{li} < q_l\}| + |\{(j, l) : \delta_{ji} < \delta_{jl} \leq r^*_j\}|$ as the communication cost of robot $i$.
Observe that if $r^*_i = 1$ for all $i \in \mathcal{V}$, then $H_i = 2$ and $C_i = 1 + |\mathcal{N}_i|$ for all $i$, rendering a fully distributed approach.
Hence, we define 
\begin{equation}
    h_i = {H_i}/{\Delta}, \quad c_i = {C_i} / {(1 + |\mathcal{N}_i|)}
\end{equation}
as metrics for the communication delay and complexity.
For comparison purposes, note that previous approaches to rigidity maintenance typically yield $h_i = 1$ and $c_i \approx \text{constant}$.}

{
\figref{fig:complexity} illustrates the behavior of these metrics by sampling $50$ \textit{IBR} frameworks per $n$ up to $100$.
Random graphs were generated using the sensing model in Section \ref{sec:robot_network_model} with $\ell_i = 0.5$ for all $i$, with robot positions uniformly sampled in the unit cube and camera axes oriented toward the barycenter of all positions.
For $n \leq 15$, the proposed protocol is comparable to the existing schemes since most robots have $0.5\!\leq\!h_i\!\leq 1$ and $c_i \leq 2$.
As $n \approx 25$, the average delay begins to improve though the communication complexity augments ($h_i \leq 0.5$ and $c_i \geq 2$ for about $75\%$ or robots).
These results reveal a trade-off between communication delay and load (at the range $15 \leq n \leq 50$).
Finally, for $n \geq 50$, most robots satisfy $h_i \leq 0.5$ showing at least a twofold improvement in delay, and $c_i \leq 2$, corresponding to a moderate communication complexity.
These results are preliminary and merit deeper investigation, which is beyond the scope of this paper.}

\begin{figure}[!tb]
    \centering
    \vspace*{5pt}
    \begin{subfigure}{0.49\columnwidth}
        \includegraphics[scale=0.85]{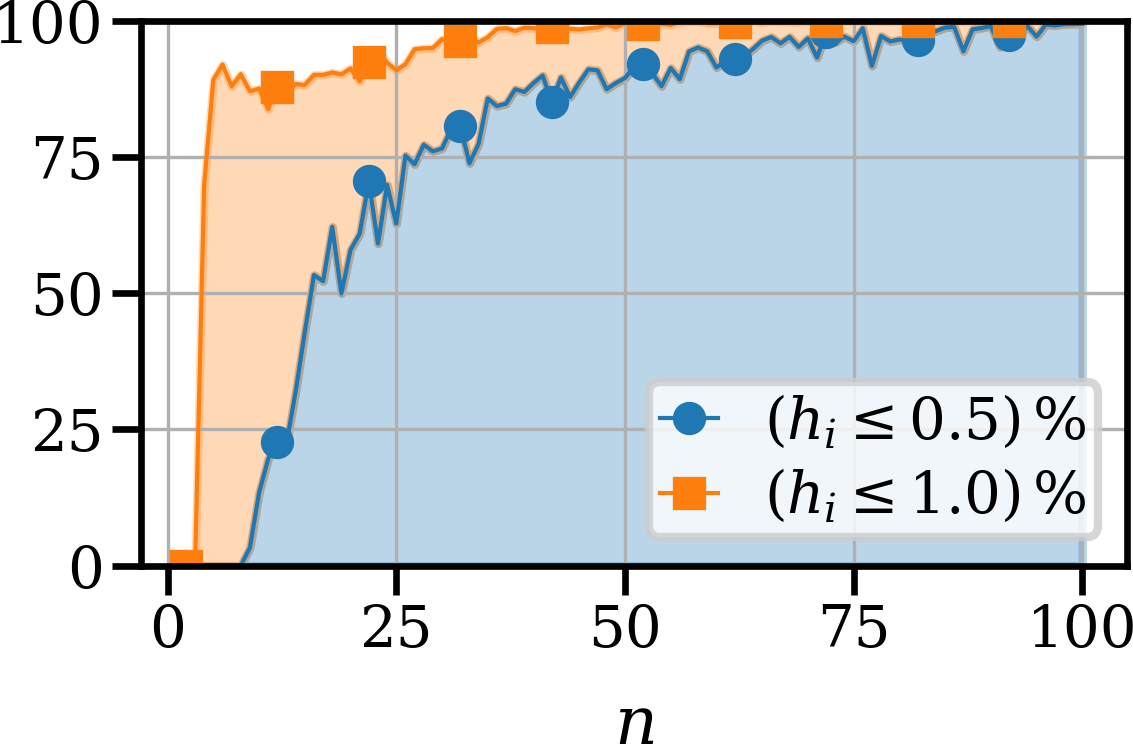}
        \caption{Delay.}
        \label{fig:bearing_delay}
    \end{subfigure}
    \begin{subfigure}{0.49\columnwidth}
        \includegraphics[scale=0.85]{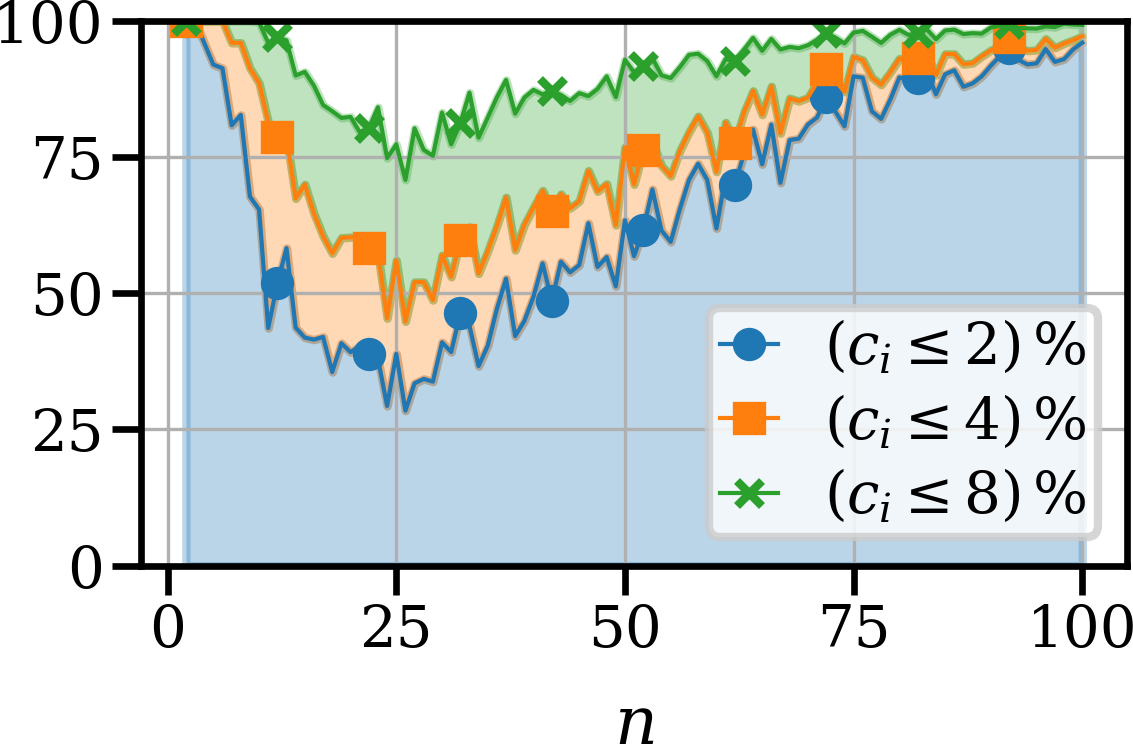}
        \caption{Complexity.}
        \label{fig:bearing_complexity}
    \end{subfigure}
    \caption{Percentage of subframeworks with $h_i \leq a$ and $c_i \leq b$ for different values of $a$ and $b$.}
    \label{fig:complexity}
\end{figure}

\section{Simulation Results}
\label{sec:simulations}
To validate the proposed control scheme, we performed simulations aimed at evaluating its performance and robustness under challenging conditions.
We employed $n = 15$ robots in $\mathbb{R}^3$, with $\ell_i = \ell_c = \SI{20}{\meter}$, $\gamma_i = 0.5$ ($\SI{60}{\degree}$ field-of-view half-angle) and $n_i = (\cos(\psi_i), \sin(\psi_i), 0)$.
As the initial realization, a framework was generated using the sensing model in Section \ref{sec:robot_network_model}, with robot positions uniformly sampled in $[0, 50]^2 \times [0, 0] \si{\cubic \meter}$ and camera axes oriented toward the barycenter of all positions.
The minimal radii of the subframeworks corresponding to this realization resulted in $r^*_i = 1$ for all $i$, yielding $C_i = 1 + |\mathcal{N}_i|$ and $H_i = 2$ for all $i$.
To maintain simplicity, the radii were held constant during the mission, even though the network's topology varied over time.
The controller parameters (see \eqref{eq:cost_ca} and \eqref{eq:cost_rm}) were $\ell_0 = \SI{1}{\meter}$ and $\lambda_0 = \num{1e-4}$, $\kappa_\mathrm{m} = 1$, $\kappa_\mathrm{c} = 0.5$ and $\kappa_\mathrm{r} = 0.1$.
As the mission objective, we propose a cooperative task involving the collection of $100$ targets randomly placed within $[0, 100]^2 \times [10, 50] \si{\cubic \meter}$.
For this purpose, the following cost function is employed, whose gradient attracts robots toward the uncollected targets, 
\begin{equation}
    J_\mathrm{m}(\boldsymbol{p}) = \sum_{i \in \mathcal{V}} f_\mathrm{m}(\zeta_i),  \tfrac{\partial J_\mathrm{m}}{\partial p_i} = f'_\mathrm{m}(\zeta_i) \frac{p_i - \tau_i}{\zeta_i},  \tfrac{\partial J_\mathrm{m}}{\partial \psi_i} = 0,
    \label{eq:target_cost}
\end{equation}
where $\tau_i$ is robot $i$'s closest uncollected target, $\zeta_i = \Vert p_i - \tau_i \Vert$, and 
$f'_\mathrm{m}$ represents the tracking velocity: it equals $\SI{1.5}{\meter\per\second}$ when $\zeta_i \leq \SI{20}{\meter}$, then decreases linearly, reaching $\SI{0}{\meter\per\second}$ at $\zeta_i \geq \SI{30}{\meter}$. 
Target $\tau_i$ is collected when $\zeta_i \leq \SI{5}{\meter}$, after which robot $i$ proceeds to the next.
This mission objective exerts a force that tends to separate the robots, which stresses the rigidity of the framework given the dynamic graph topology.
The distributed network localization algorithm from \cite{Zhao2016AUT} was implemented to enable robots to estimate their relative positions up to a scale using bearing measurements (which are communicated between neighboring robots).
In general, an additional range measurement between a pair of robots suffices to recover the correct scale, as shown in \cite{Schiano2016}.
The application considered requires that robots estimate their relative positions with respect to targets, see \eqref{eq:target_cost}. Hence, we assume that
two robots measure their relative positions with respect
to a target, which also ensures the correct scale.

\figref{fig:snapshots} shows the state of the network at different stages; example videos can
be found in \cite{SubframeworkVideo}.
Filled and empty diamonds represent uncollected and collected targets, respectively.
\figref{fig:simulation} presents performance metrics as follows. \figref{fig:eigenvalues} shows the evolution of the rigidity eigenvalues: the shadowed area is bounded by the min and max values, while the dashed line represents the rigidity of the framework.
After an initial rapid improvement of the subframework's rigidity eigenvalues, the controller manages to keep them above $\lambda_0$ despite the number of edges (measurements) decreasing over time as the robots disperse to collect targets.
This results from the controller’s effort to maximize rigidity by optimizing the robots' positions, resulting in triangular patterns.
It is noteworthy that, although the rigidity eigenvalue of the framework is only guaranteed to remain positive, its time evolution closely mirrors the subframeworks’ rigidity eigenvalues.
\figref{fig:targets} displays how the number of targets that were collected by robots steadily increases as the mission progresses. 
\figref{fig:distances} shows that the controller manages to avoid inter-robot collisions by maintaining all distances above $\ell_0$. 
\figref{fig:diameters} presents the diameters of the subframeworks (which overlap within the shadowed area), and the framework's diameter (dashed line).
Note that as robots disperse, the framework's diameter increases, but the diameter of the subframeworks stays below $2$, demonstrating the practicality of the proposed scheme.

\begin{figure}[!tb]
    \centering
    \vspace*{5pt}
    \includegraphics[scale=1.05]{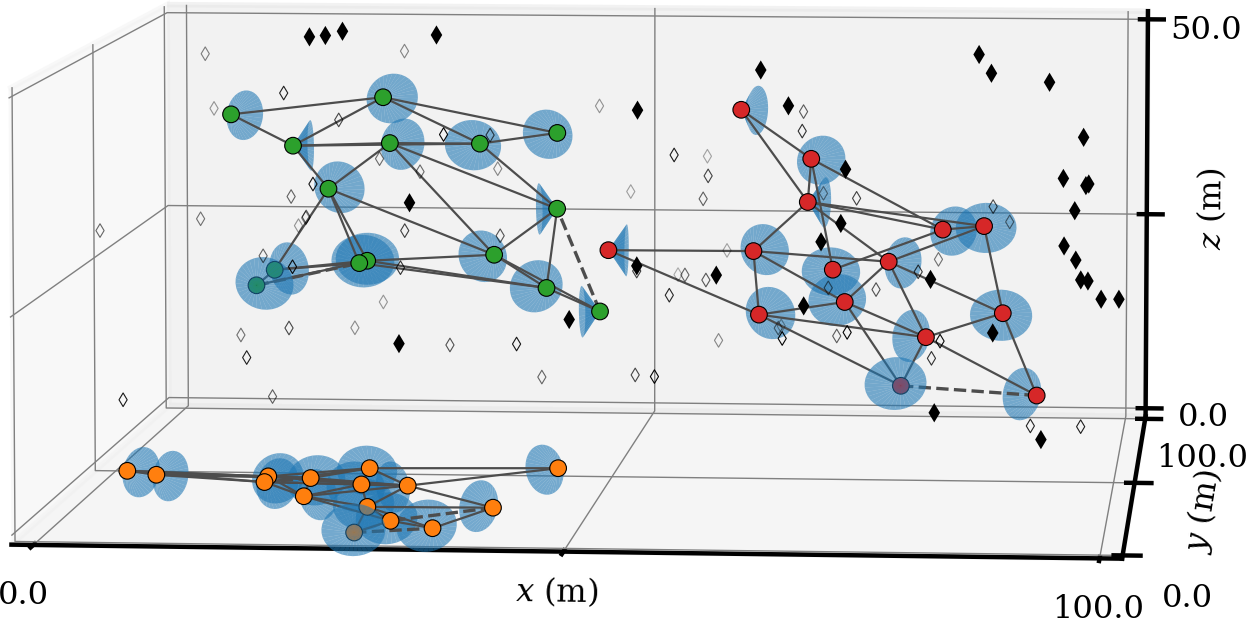}
    \caption{Snapshots at $t = \SI{0}{\second}$ (orange), $\SI{100}{\second}$ (green) and  $\SI{300}{\second}$ (red). Solid lines represent $\mathcal{E}$ while dashed lines $\mathcal{E}_c \setminus \mathcal{E}$.}
    \label{fig:snapshots}
\end{figure}

\begin{figure}[!tb]
    \centering
    \begin{subfigure}{0.49\columnwidth}
        \includegraphics[scale=0.85]{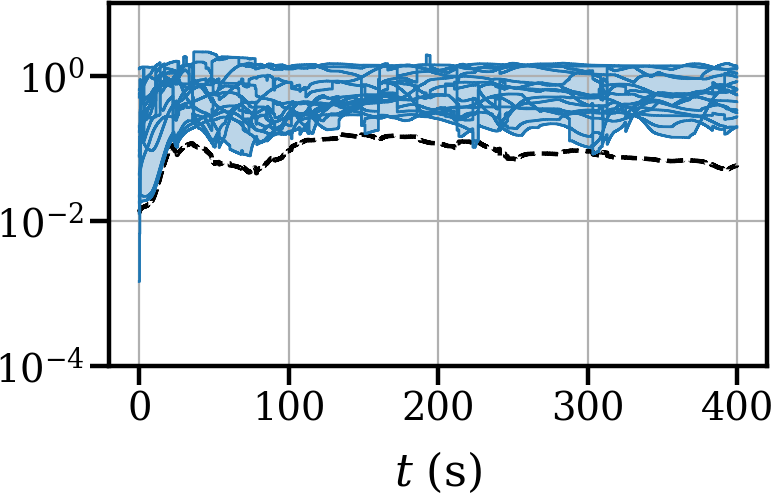}
        \caption{Rigidity eigenvalues vs. $t$.}
        \label{fig:eigenvalues}
    \end{subfigure}
    \begin{subfigure}{0.49\columnwidth}
        \includegraphics[scale=0.85]{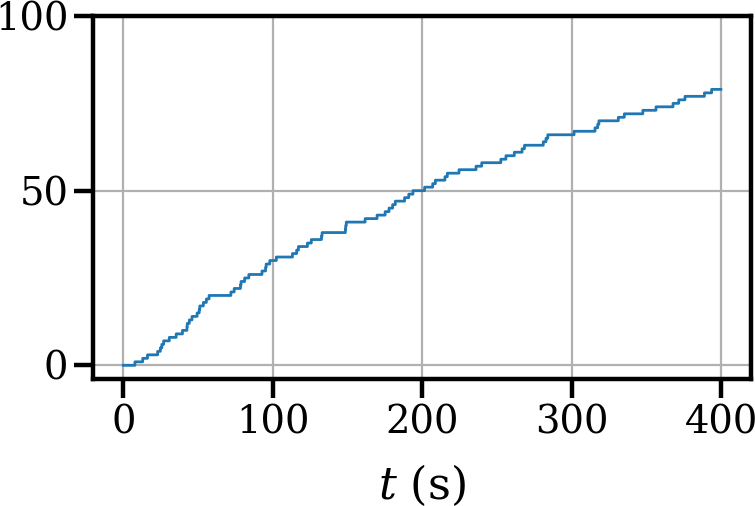}
        \caption{Collected targets vs. $t$.}
        \label{fig:targets}
    \end{subfigure}
    \begin{subfigure}{0.49\columnwidth}
        \includegraphics[scale=0.85]{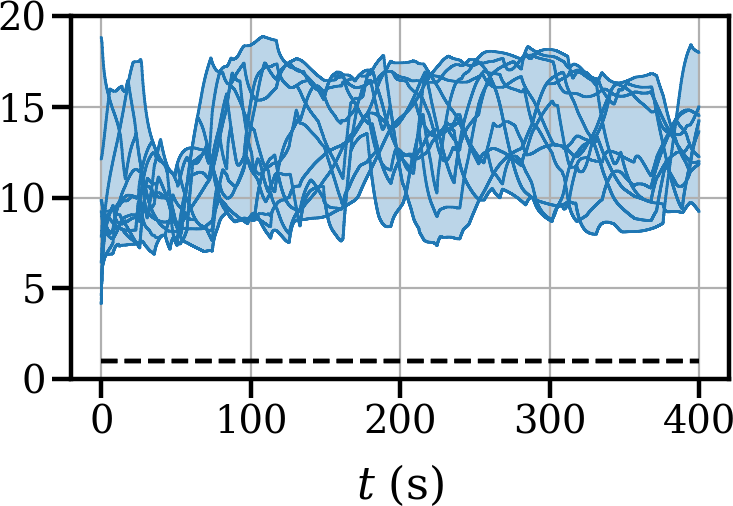}
        \caption{Inter-robot distances vs. $t$.}
        \label{fig:distances}
    \end{subfigure}
    \begin{subfigure}{0.45\columnwidth}
        \includegraphics[scale=0.85]{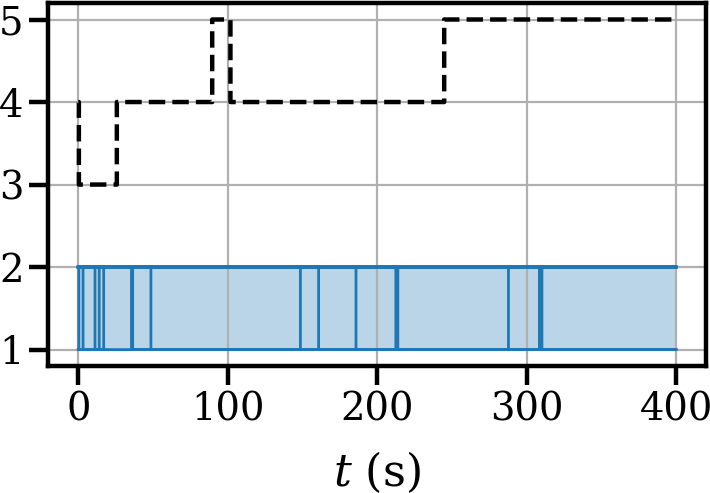}
        \caption{Diameters vs. $t$.}
        \label{fig:diameters}
    \end{subfigure}
    \caption{Simulation performance metrics.}
    \label{fig:simulation}
\end{figure}

\section{Conclusions and Future Work}
\label{sec:conclusions}

We presented a subframework-based approach for studying bearing rigidity, effectively translating this global property into multiple local ones.
The rigidity eigenvalues of the subframeworks were employed to design a decentralized rigidity maintenance controller that avoids collisions between robots.
Both objectives are achieved as soft constraints, through the minimization of cost functions.
Simulations were performed to demonstrate the controller's capabilities.
The proposed strategy improves the delay issues of global estimators, as information is confined to subframeworks---typically much smaller than the full framework.
However, it incurs an increased communication complexity that can compromise scalability---unless subframeworks are small.
Simulations were conducted to explore this trade-off between communication delay and complexity.

{
As future research directions, we plan to extend our approach to consider the lack of a common reference frame to express bearing measurements.
This would require adapting the subframework-based tools to also account for orientation as part of the robot’s uncertain state.}
Also, we highlight the relevance of exploring ways to reduce communication complexity, while preserving the small delays.
Three different approaches arise.
Firstly, the design of optimal initial formations that minimize subframework sizes.
Secondly, developing a decentralized protocol for dynamically reducing the radii of the subframeworks.
Finally, a sparser decomposition---comprising only a minimal set of subframeworks---could significantly improve the overall efficiency of the system.

\bibliographystyle{IEEEtran}
\bibliography{refs}

\end{document}